\documentclass[10pt,twocolumn,twoside]{IEEEtran}
\usepackage{amsopn}
\usepackage{amsthm}
\usepackage{etex}
\usepackage{spconf,graphicx}
\usepackage{endnotes}
\usepackage{epsfig,psfrag}
\usepackage{pst-all}
\usepackage{amssymb,amsfonts,upref,cite,epsf,color,bm}
\usepackage{color}
\usepackage{url}
\usepackage{amsmath}
\usepackage{graphicx}
\usepackage{subfig}
\usepackage{pstricks}
\usepackage{calc}
\usepackage{booktabs}
\usepackage{tikz}
\usepackage{pgfplots}
\newcommand\defeq{:=}

\usepackage{algorithm}% http://ctan.org/pkg/algorithms
\usepackage{algpseudocode}% http://ctan.org/pkg/algorithmicx
\floatname{algorithm}{Algorithm}
\algnewcommand\algorithmicinput{\textbf{Input:}}
\algnewcommand\INPUT{\item[\algorithmicinput]}
\algnewcommand\algorithmicoutput{\textbf{Output:}}
\algnewcommand\OUTPUT{\item[\algorithmicoutput]}

\DeclareMathOperator*{\argmin}{arg\;min}

\newcommand\vect[1]{\mathbf #1}
\newcommand{\xsig}{x[\cdot]}  
\newcommand{\xsigval}[1]{x[#1]}

\newcommand{\vx}{\vect{x}}

\newcommand{\mW}{\mathbf{W}}
\newcommand{\measlen}{M}

\newcommand{\signalsize}{N}
\newcommand{\siglen}{N}

\newcommand{\graphsigs}{\mathbb{R}^{\mathcal{V}}} 

\newcommand{\edgesupport}{\mathcal{S}}
\newcommand{\edges}{\mathcal{E}}
\newcommand{\cluster}{\mathcal{C}}
\newcommand{\partition}{\mathcal{F}}
\newcommand{\nodes}{\mathcal{V}}
\newcommand{\graph}{\mathcal{G}}
\newcommand{\samplingset}{\mathcal{M}}
\newcommand{\clusteredsig}{x_{c}[\cdot]}

\newtheorem{theorem}{Theorem}%[section]
\newtheorem{definition}[theorem]{Definition}
\newtheorem{lemma}[theorem]{Lemma}

\newcommand{\edge}[2]{\{#1,#2\}}

\usepackage{setspace}

\title{The Network Nullspace Property for Compressed Sensing of Big Data over Networks}
\name{Alexander Jung$^{1}$, Madelon Hulsebos$^{2}$}
\address{\normalsize $^1$Department of Computer Science, Aalto University, Finland; firstname.lastname(at)aalto.fi \\[-0.5mm]
               \normalsize $^2$Department of Computer Science, Delft University of Technology, Delft, The Netherlands
}

\begin{document}
%\thanks{\hspace*{-5mm}The work of ??? was supported by ???.} 
	\maketitle
\begin{abstract}
We present a novel condition, which we term the network nullspace property, which ensures 
accurate recovery of graph signals representing massive network-structured datasets 
from few signal values. 
%r convex optimization methods to accurately 
%learn labels which form smooth graph signals. %recover clustered graph signals. 
The network nullspace property couples the cluster structure of the underlying 
network-structure  with the geometry of the sampling set. Our results can be used 
to design efficient sampling strategies based on the network topology. 
% suggest 
%how to design the  important application of our 
%results is the design of efficient sampling strategies based which is guided by the 
%network topology. 
%in order to obtain the most information about the entire graph signal. 
% Assuming the true underlying graph signal to be constant over well connected subset of nodes (clusters), 
%we derive a sufficient condition on the sampling set and network structure 
%such that the proposed convex method is accurate. 
%This condition, which we coin the \emph{network nullspace property}, 
%characterizes which nodes of the graph should be sampled 
%in order to retain the full information about the underlying graph signal. 
\end{abstract}

\begin{keywords} compressed sensing, 
%label propagation, 
%cosparse analysis model, 
big data, 
semi-supervised learning, 
complex networks, 
convex optimization %, nullspace property
% , subsampling
\end{keywords} 

% no keywords
% \vspace*{-1mm}
%%%%%%%%%%%%%%%%%%%%%%%%%%%%%%%%%%%
\section{Introduction}
 \label{sec_intro}

%A multitude of convex optimization methods have been recently proposed for machine learning based on 
%massive network-structure datasets (``big data over networks'') \cite{Zhu02learningfrom,SharpnackJMLR2012,TrendGraph,JungSpawc2016,HannakAsilomar2016,JungHero2016}. 
%While most work focused on efficient implementation of those learning methods, the characterization 
%of conditions which guarantee those methods to be accurate has received little attention so far. 

A recent line of work proposed efficient convex optimization methods for recovering graph signals 
which represent label information of network structured datasets (cf. \cite{JungSpawc2016,HannakAsilomar2016}). 
These methods rest on the hypothesis that the true underlying graph signal is nearly constant over well-connected 
subsets of nodes (clusters). 

In this paper, we introduce a novel recovery condition, termed the network nullspace property (NNSP), which 
guarantees convex optimization to accuratetly recovery of clustered (``piece-wise constant'') graph signals from knowledge of its 
values on a small subset of sampled nodes. The NNSP couples the clustering structure of the underlying data graph to 
the locations of the sampled nodes via interpreting the underlying graph as a flow network. 

The presented results apply to an arbitrary partitioning, but are most useful for a partitioning such that 
nodes in the same cluster are connected with edges of relatively large weights, whereas edges between clusters have low weights. 
Our analysis reveals that if cluster boundaries are well-connected (in a sense made precise) 
to the sampled nodes, then accurate recovery of clustered graph signals is possible by solving 
a convex optimization problem. 

Most of the existing work on graph signal processing applies spectral graph theory to define a notion 
of band-limited graph signals, e.g.\ based on principal subspaces of the graph Laplacian matrix, 
as well as sufficient conditions for recoverability, i.e., sampling theorems, for those signals \cite{MaruesSeg2016,ChenVarma2015}. 
In contrast, our approach does not rely on spectral graph theory, but involves structural (connectivity) 
properties of the underlying data graph. Moreover, we consider signal models which amount to 
clustered (piece-wise constant) graph signals. These models, which have also been used in \cite{ChenVarma2016} 
for approximating graph signals arising in various applications. 

The closest to our work is \cite{SharpnackJMLR2012,TrendGraph}, 
which provide sufficient conditions such that a variant of the Lasso method (network Lasso) 
accurately recovers clustered graph signals from noisy observations. However, in contrast to 
this line of work, we assume the graph signal values are only observed on a small subset of nodes. 

The NNSP is closely related to the network compatibility (NCC)
condition, which has been introduced by a subset of the authors in \cite{NlassoFrontiers2018} for analyzing 
the accuracy of network Lasso. The NCC is a stronger condition in the sense that once the NCC is satisfied, 
the NNSP is also guaranteed to hold. 

\section{Problem Formulation}
\label{sec_setup}

Many important applications involve massive heterogeneous datasets comprised 
heterogeneous data chunks, e.g., mixtures of audio, video and text data \cite{BigDataNetworksBook}. 
Moreover, datasets typically contain mostly unlabeled data points; only a small fraction 
is labeled data. An efficient strategy to handle such heterogenous datasets is to organize 
them as a network or data graph whose nodes represent individual data points. 

%Many applications naturally suggest a notion of similarity between individual 
%data points, e.g., the profiles of befriended social network users or greyscale 
%values of neighbouring image pixels. Within image processing, we can represent $2$D images 
%by grid graphs with a particular node $i \!\in\! \nodes$ representing a (super-)pixel \cite{JungSpawc2016,HannakAsilomar2016}. 
%\vspace*{-3mm}
\subsection{Graph Signal Representation of Data}
%\vspace*{-1mm}

In what follows we consider datasets which are represented by a weighted data graph $\graph\!=\!(\nodes,\edges,\mathbf{W})$ 
with nodes $\nodes\!=\!\{1,\ldots,\signalsize\}$, each node representing an individual data point. These nodes are connected 
by edges $\edge{i}{j}\!\in\!\edges$. 
In particular, given some application-specific notion of similarity, the edges of the data graph $\graph$ connect similar 
data points $i,j \in \nodes$ by an edge $\edge{i}{j}\!\in\!\edges$. 
In some applications it is possible to quantify the extent to which data points 
are similar, e.g., via the distance between sensors in a wireless sensor network \cite{zhao2004wsn}. 
Given two similar data points $i,j\!\in\!\nodes$, we quantify the strength 
of their connection $\edge{i}{j} \in \edges$ by a non-negative edge weight $W_{i,j}\!\geq\!0$ which we collect 
in the symmetric weight matrix $\mathbf{W} \in \mathbb{R}_{+}^{\signalsize \times \signalsize}$.  
%Given a subset $\edgeset \!\subseteq\! \edges$ of edges, by a slight abuse of notation, we use the same 
%symbol $\edgeset$ to indicate the subset of nodes connected by the edges in $\edgeset$. 

In what follows we will silently assume that the data graph $\graph$ is oriented by declaring for each 
edge $\edge{i}{j} \in \edges$ one node as the head $e^{+}$ and the other node as the tail $e^{-}$. 
For the oriented data graph we define the directed neighbourhoods of a node $i\!\in\!\nodes$ 
as $\mathcal{N}^{+}(i)\!\defeq\!\{ j\!\in\!\mathcal{N}(i)\!:\! e=\edge{i}{j} \in \edges \mbox{, and } i=e^{+}\}$ and 
$\mathcal{N}^{-}(i)\!\defeq\!\{ j\!\in\!\mathcal{N}(i)\!:\! e=\edge{i}{j} \in \edges \mbox{, and } i=e^{-}\}$. 
We highlight that the orientation of the data graph $\graph$ is not related to any intrinsic property of the 
underlying data set. In particular, the weight matrix $\mathbf{W}$ is symmetric since the weights $W_{i,j}$ 
are associated with undirected edges $\{i,j\} \in \edges$. However, using an (arbitrary but fixed) orientation 
of the data graph will be notationally convenient in order to formulate our main results. 

Beside the edges structure $\edges$, network-structured datasets typically also carry label information 
which induces a graph signal defined over $\graph$. We define a graph signal $x[\cdot]$ 
over the graph $\graph=(\nodes,\edges,\mW)$ as a mapping $\nodes \rightarrow \mathbb{R}$, which associates 
(labels) every node $i\!\in\!\nodes$ with the signal value $x[i] \!\in\! \mathbb{R}$. 
In a supervised machine learning application, the signal values $x[i]$ might represent 
class membership in a classification problem or the target (output) value in a regression problem. 
We denote the space of all graph signals, which is also known as the vertex space (cf.\ \cite{DiestelGT}), 
by $\graphsigs$. 

\vspace{-0.2cm}
\subsection{Graph Signal Recovery}
\label{equ_gsr_sec}
\vspace*{-1mm}

We aim at recovering (learning) a graph signal $x[\cdot] \in \graphsigs$ defined over 
the data graph $\graph$, from observing its values $\{ x[i] \}_{i \in \samplingset}$ 
on a (small) sampling set $\samplingset\defeq\{i_{1},\ldots,i_{\measlen}\} \subseteq \nodes$, where typically $\measlen \ll \signalsize$. 

The recovery of the entire graph signal $\xsig$ from the incomplete information provided by 
the signal samples $\{x[i]\}_{i \in \samplingset}$ is possible under a smoothness assumption, 
which is also underlying many supervised machine learning methods \cite{SemiSupervisedBook}. 
This smoothness assumption requires the signal values or labels of data points which are close,  
with respect to the data graph topology, to be similar. 
More formally, we expect the underlying graph signal $\vx[\cdot] \in \graphsigs$ 
to have a relatively small \emph{total variation} (TV) 
\begin{equation} 
\| x[\cdot] \|_{\rm TV} \defeq  \sum_{ \edge{i}{j} \in \edges} W_{i,j}  | x[i]\!-\!x[j]|. \nonumber 
\end{equation}
The total variation of the graph signal $x[\cdot]$ obtained  
over a subset $\mathcal{S}\subseteq \edges$ of edges is denoted $\| x[\cdot] \|_{\mathcal{S}} \defeq \sum_{\edge{i}{j} \in \mathcal{S}} W_{i,j}  | x[j]\!-\!x[i]|$.

Some well-known examples of smooth graph signals include low-pass signals in 
digital signal processing where time samples at adjacent time instants are strongly 
correlated and  close-by pixels in images tend to be coloured likely. The class of graph signals 
with a small total variation are sparse in the sense of changing significantly over few edges only. 
In particular, if we stack the signal differences $x[i]-x[j]$ (across the edges $\edge{i}{j} \in \edges$) into a big 
vector of size $|\edges|$, then this vector is sparse in the ordinary sense of having 
only few significantly large entries \cite{Don06}. 

In order to recover a signal with small TV $\| x[\cdot] \|_{\rm TV}$,  
from its signal values $\{ x[i]\}_{i \in \samplingset}$, 
a natural strategy is %the optimization problem  
\begin{align} 
\hat{x}[\cdot] & \!\in\! \argmin_{\tilde{x}[\cdot] \in \graphsigs} \| \tilde{x}[\cdot] \|_{\rm TV} 
\hspace*{1mm} \mbox{s.t.} \hspace*{1mm} \tilde{x}[i]\!=\!x[i] \mbox{ for all } i\!\in\!\samplingset.  \label{equ_semi_sup_learning_problem}
\end{align}
%As the notation indicates, there might be multiple solutions $\hat{x}[\cdot]$ for the learning 
%problem \eqref{equ_semi_sup_learning_problem}. However, any solution 
%$\hat{x}[\cdot]$ of \eqref{equ_semi_sup_learning_problem} 
%is characterized by: (i) it is consistent with the observed samples $\{x[i]\}_{i \in \samplingset}$ 
%and (ii) it has minimum TV among all such graph signals. 
There exist highly efficient methods for solving convex optimization problems of 
the form  \eqref{equ_semi_sup_learning_problem} 
(cf.\ \cite{ZhuAugADMM,JungHero2016,pock_chambolle} and the references therein). 
%In particular, the application of the primal-dual method by Pock and Chambolle 
%\cite{} results in a highly scalable sparse label propagation (SLP) algorithm \cite{}. 

%\vspace*{-3mm}
\section{Recovery Conditions} 
\label{sec_main_results} 
%\vspace*{-1mm}

The accuracy of any learning method based on solving  \eqref{equ_semi_sup_learning_problem} 
depends on the deviations between the solutions $\hat{x}[\cdot]$ of the optimization problem \eqref{equ_semi_sup_learning_problem} 
and the true underlying graph signal $x[\cdot] \in \graphsigs$. 
In what follows, we introduce the network nullspace condition as a sufficient condition on 
the sampling set and graph topology such that any solution $\hat{x}[\cdot]$ of \eqref{equ_semi_sup_learning_problem} 
accurately resembles an underlying clustered (piece-wise constant) graph signal of the form (cf. \cite{ChenVarma2016}) 
\begin{equation}
\label{equ_def_clustered_signal_model}
x[i] = \sum_{\cluster \in \partition} a_{\cluster} \mathcal{I}_{\cluster}[i] \quad\quad \mbox{with } 
\mathcal{I}_{\cluster}[i] \defeq \begin{cases} 1 \mbox{ for } i \in \cluster \\ 0 \mbox{ else.}  \end{cases}
\end{equation}

The signal model \eqref{equ_def_clustered_signal_model} is defined using a fixed partition $\partition= \{ \cluster_{1},\ldots,\cluster_{|\partition|} \}$ of the entire data 
graph $\graph$ into disjoint clusters $\cluster_{l} \subseteq \nodes$. 
The signal model \eqref{equ_def_clustered_signal_model} has been studied in \cite{ChenVarma2016}, 
where it was demonstrated that it allows, compared to band-limited graph signal models, 
for more accurate approximation of datasets obtained from weather stations.

While our analysis applies to an arbitrary partition $\partition$, our results are most useful for partitions 
where the nodes within clusters $\cluster_{l}$ are connected by many edges with large weight, while nodes of different 
clusters are loosely connected by few edges with small weights. Such reasonable partitions can be obtained 
by one of the recently proposed highly scalable clustering methods, e.g., \cite{Spielman_alocal,Fortunato2009}. 

We highlight that the knowledge of the partition is only required for the analysis of signal recovery methods 
(such as sparse label propagation \cite{JungHero2016}), which are based on solving the recovery problem 
\eqref{equ_semi_sup_learning_problem}, However, for the actual implementation of those methods, 
as the recovery problem \eqref{equ_semi_sup_learning_problem} 
itself does not involve the partition.  

We will characterize a partition $\partition$ by its boundary 
\begin{equation}
\label{equ_def_boundary}
\partial \partition \!\defeq\! \{ \{i,j\} \!\in\! \edges\!:\! i \in \cluster_{l}, j \in \cluster_{l'} \mbox{, with } l \!\neq\!l' \} \subseteq \edges, 
\end{equation}
which is the set of edges connecting nodes from different clusters. 
We highlight that the recovery problem \ref{equ_semi_sup_learning_problem} 
does not require knowledge of the partition $\partition$. 
Rather, the partition $\partition$ and corresponding signal model 
\eqref{equ_def_clustered_signal_model} is only used for analyzing the solutions of \eqref{equ_semi_sup_learning_problem}. 

\subsection{Network Nullspace Property}
Consider a clustered graph signal $\xsig \in \graphsigs$ of the form \eqref{equ_def_clustered_signal_model}. We observe its values $x[i]$ at 
the sampled nodes $i \in \samplingset$ only. In order to have any chance for recovering the complete signal only from 
the samples $\{ x[i] \}_{i \in \samplingset}$ we have to restrict the nullspace of the sampling set, which we define as 
\begin{equation}
\label{equ_def_kernel_sampling_set}
\mathcal{K}(\samplingset) \defeq \{ \tilde{x}[\cdot] \in \graphsigs :  \tilde{x}[i] = 0 \mbox{ for all } i \in \samplingset \}. 
\end{equation}
Thus, the nullspace $\mathcal{K}(\samplingset)$ contains exactly those graph signals which vanish at all nodes of the sampling set $\samplingset$. 
Clearly, we have no chance in recovering any signal $\hat{x}[\cdot]$ which belongs to the nullspace $\mathcal{K}(\samplingset)$ 
as it can not be distinguished from the all-zero signal $\tilde{x}[i]=0$, for all nodes $i\!\in\!\nodes$, 
which result in exactly the same (vanishing) measurements $\tilde{x}[i]\!=\!\hat{x}[i]\!=\!0$ for all $i\!\in\!\samplingset\subseteq\nodes$.   

In order to define the network nullspace property which characterizes the solutions of the recovery 
problem \eqref{equ_semi_sup_learning_problem}, we need the notion of a \emph{flow with demands} \cite{KleinbergTardos2006}. 
\begin{definition} 
Given a graph $\graph\!=\!(\nodes,\edges,\mathbf{W})$, a flow with demands $g[i]\!\in\!\mathbb{R}$, for $i\!\in\!\nodes$, is a mapping 
$f[\cdot]: \edges \rightarrow \mathbb{R}_{+}$ satisfying the conservation law 
%\vspace*{-2mm}
\begin{equation}
\label{equ_define_demand_now}
\sum_{j \in \mathcal{N}^{+}(i)}\hspace*{-2mm} f[\edge{i}{j}] \!-\! \hspace*{-2mm}\sum_{j \in \mathcal{N}^{-}(i)} \hspace*{-2mm}f[\edge{i}{j}] = g[i] 
%\vspace*{-2mm}
\end{equation}
at every node $i \!\in\! \nodes$.
\end{definition} 
For a more detailed discussion of the concept of network flows, we refer  to \cite{KleinbergTardos2006}. 
In this paper, we use the concept of network flows in order to characterize the connectivity properties 
or topology of a data graph $\graph\!=\!(\nodes,\edges,\mathbf{W})$ by 
interpreting the edge weights $W_{i,j}$ as capacity constraints that limit the amount of flow along the edge $\edge{i}{j} \in \edges$ which 
connects the data points $i,j \in \nodes $. 

In particular, the notion of network flows with demands allows to adapt the nullspace property, 
introduced within the theory of compressed sensing \cite{RauhutFoucartCS,EldarBookSamplingTheory} for 
sparse signals, to the problem of recovering clustered graph signals (cf. \eqref{equ_def_clustered_signal_model}).  

\begin{definition}
\label{def_sampling_set_resolves}
Consider a partition $\partition=  \{ \cluster_{1},\ldots,\cluster_{|\partition|} \}$ of 
pairwise disjoint subsets of nodes (clusters) $\cluster_{l} \subseteq \nodes$ and 
a set of sampled nodes $\samplingset \subseteq \nodes$.
The sampling set $\samplingset$ is said to satisfy the network nullspace property relative 
to the partition $\partition$, denoted NNSP-$\partition$, if for any signature $\sigma_{e} \in \{-1,1\}^{\partial \partition}$, which 
assigns the sign $\sigma_{e}$ to a boundary edge $e\in \partial \partition$, 
there is a flow $f[e]$ 
\begin{itemize} 
\item with demands $g[i]=0$, for $i \notin \samplingset$, 
\item its values satisfy    
\begin{align}
\label{equ_def_flows_NNSP}
f[e] &\!=\! \kappa \sigma_{e} W_{e} \mbox{ for } e\!\in\!\partial \partition \mbox{ with some $\kappa >1$, and }  \nonumber \\
f[e]  &\!\leq\!  W_{e} \mbox{ for } e\!\in\!\edges \setminus \partial \partition.
%g[i]  &\!=\! 0 \mbox{ for every node } i \!\notin\!   \samplingset. 
\end{align} 

\end{itemize}
\end{definition} 
%Note that Definition \ref{def_sampling_set_resolves} allows for arbitrary demands $g[i]$ at a sampled node $i \in \samplingset$. 
%We highlight that within Definition \ref{def_sampling_set_resolves}, we interpret the data graph $\graph$ as a flow network with the edge weights $W_{e}$ 
%being capacity constraints for some abstract flow over the edges of the data graph. 
It turns out that a sampling set $\samplingset$ satisfies NNSP-$\partition$ for a given partition $\partition$ of 
the data graph, then the nullspace $\mathcal{K}(\samplingset)$ (cf.\ \eqref{equ_def_kernel_sampling_set}) of the sampling set 
cannot contain a non-zero clustered graph signal of the form \eqref{equ_def_clustered_signal_model}. 

A naive verification of the NNSP involves a search over all signatures, whose number is around $2^{| \partial \partition|}$, which might be 
% a naive verification of the NNSP for 
%a given large data graph becomes 
intractable for large data graphs. % The situation is somewhat similar to compressed sensing of 
%sparse vectors, where the verification of nullspace property for a give matrix is NP-hard in general 
%\cite{Tillmann2014}. 
However, similar to many results in compressed sensing, we expect using probabilistic 
models for the data graph to render the verification of NNSP tractable \cite{RauhutFoucartCS}. 
In particular, we expect that probabilistic statements about how likely the NNSP is satisfied for random 
data graphs (e.g., conforming to a stochastic block model) can be obtained easily.

%We will make Definition \ref{def_sampling_set_resolves} more transparent 
%by stating an easy-to-check sufficient condition for the sampling set $\samplingset$ 
%and graph structure such that NNSP holds (cf.\ Lemma \ref{lem_suff_cond_NNSP} below). 
%%But first, we need to introduce a model for clustered graph signals. 
%\begin{lemma}
%\label{lem_suff_cond_NNSP}
%Consider a data graph $\graph$ containing the sampled nodes $\samplingset$. 
%The nodes of $\graph$ are partitioned into disjoint clusters $\partition=\{\cluster_{1},\ldots\}$. 
%If each boundary edge $\{i,j\} \in \partial \partition$ (cf.\ \eqref{equ_def_boundary}) with  $i\!\in\!\cluster_{l}$, $j \!\in\! \cluster_{l'}$ is 
%connected to sampled nodes, i.e., $\{m,i\}\!\in\!\edges$ and $\{n,j\}\!\in\!\edges$ with $m \!\in\! \samplingset\!\cap\!\cluster_{l}$,  
%$n\!\in\!\samplingset\!\cap\!\cluster_{l'}$, and weights $W_{m,i}, W_{n,j} \geq 2 W_{i,j}$, then 
%NNSP-($\samplingset,\partition$) holds. 
%\end{lemma} 

\subsection{Exact Recovery of Clustered Signals}
Now we are ready to state our main result, i.e., the NNSP ensures the solution 
\eqref{equ_semi_sup_learning_problem} to be unique and to coincide with the true underlying 
clustered graph signal of the form \eqref{equ_def_clustered_signal_model}. 
\begin{theorem} 
\label{main_thm_exact_sparse}
Consider a clustered graph signal $\clusteredsig\!\in\!\mathcal{X}$ (cf.\ \eqref{equ_def_clustered_signal_model})
which is observed only at the sampling set $\samplingset \subseteq \nodes$. 
If the sampling set $\samplingset$ satisfies NNSP-$\partition$, then the solution of \eqref{equ_semi_sup_learning_problem} 
is unique and coincides with $\clusteredsig$. 
\end{theorem} 
Thus, if we sample a clustered graph signal $\xsig$ (cf.  \eqref{equ_def_clustered_signal_model}) 
on a sampling set which satisfies NNSP-$\partition$, we can expect convex 
recovery algorithms (which are based on solving \eqref{equ_semi_sup_learning_problem}) to accurately recover $\xsig$. 

{\bf A partial converse.} The recovery condition provided by Theorem \ref{main_thm_exact_sparse} is essentially tight, 
i.e., if the sampling set does not satisfy NNSP-$\partition$, then the are solutions to \eqref{equ_semi_sup_learning_problem} which are 
different from the true underlying clustered graph signal.

Consider a clustered graph signal 
\begin{equation} 
\label{equ_partial_converse_graph_sig}
x_{\rm c}[i]= 1\cdot \mathcal{I}_{\mathcal{C}_{1}}[i] + 2 \cdot \mathcal{I}_{\mathcal{C}_{2}}[i]
\end{equation} 
defined over a chain graph $\graph_{\rm chain}$ containing an even number $N$ of nodes (Figure \ref{fig_twoclusterchain}). 
We partition the graph into two equal-sized clusters $\partition_{\rm c} = \{\cluster_{1},\cluster_{2}\}$ with $\cluster_{1} = \{1,\ldots,N/2\}$ 
and $\cluster_{2} = \{N/2+1,\ldots,N \}$. 
The edges within clusters are connected by edges with unit weight, while the single edge $\{ N/2, N/2+1\}$ connecting the two clusters has 
weight $1/\delta$. Let us assume that we sample the graph signal $x_{\rm c}[i]$ on the sampling set $\samplingset_{\rm c}= \{1,N \}$. 

For any $\delta >1$, the sampling set $\samplingset_{\rm c}$ satisfies the NNSP-$\partition_{\rm c}$ with $\kappa = \delta >1$ (cf. \eqref{equ_def_flows_NNSP}). 
Thus, as long as the boundary edge has weight $1/\delta$ with $\delta >1$, Theorem \ref{main_thm_exact_sparse} guarantees that the true clustered graph signal $x_{\rm c}[i]$ can 
be perfectly recovered via solving \eqref{equ_semi_sup_learning_problem}. 

If, on the other hand, the weight of the boundary edge is $1/ \delta$ with some $\delta \leq 1$, 
then the sampling set $\samplingset_{\rm c}$ does not satisfy the NNSP-$\partition_{\rm c}$. In this case, 
as can be verified easily, the true graph signal $x_{\rm c}[i]$ is not the unique solution to \eqref{equ_semi_sup_learning_problem} 
anymore. Indeed, for $\delta \leq 1$ it can be shown that the graph signal \eqref{equ_partial_converse_graph_sig} 
has a TV norm at least as large as the graph signal $x'[i] = \big( 1\!-\!\frac{i-1}{N-1} \big) x_{\rm c}[1] + \frac{i-1}{N-1}x_{\rm c}[N]$, which linearly interpolates 
between the sampled signal values $x_{\rm c}[1]$ and $x_{\rm c}[N]$. 

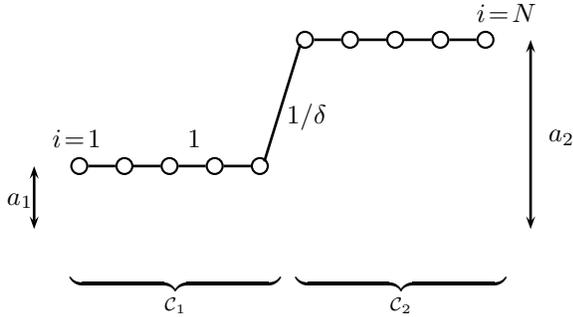
\begin{figure}
\begin{pspicture}(-1.5,-1)(5.5,2.5)
\psset{unit=1.2cm}
\psline[linewidth=1pt]{<->}(-0.5,0)(-0.5,0.7)
\pscircle(0,0.7){0.1}
\rput[tl](1.2,1.1){$1$}
\psline[linewidth=1pt]{}(0.1,0.7)(0.4,0.7)
%\psline[linewidth=1pt]{}(0,0)(0,0.6)
\pscircle(0.5,0.7){0.1}
\psline[linewidth=1pt]{}(0.6,0.7)(0.9,0.7)
\pscircle(1,0.7){0.1}
\psline[linewidth=1pt]{}(1.1,0.7)(1.4,0.7)
\pscircle(1.5,0.7){0.1}
\psline[linewidth=1pt]{}(1.6,0.7)(1.9,0.7)
\pscircle(2,0.7){0.1}
\psline[linewidth=1pt]{}(2.05,0.75)(2.45,2.05)
\rput[tl](2.3,1.4){$1/\delta$}
\pscircle(2.5,2.1){0.1}
\psline[linewidth=1pt]{}(2.6,2.1)(2.9,2.1)
\pscircle(3,2.1){0.1}
\psline[linewidth=1pt]{}(3.1,2.1)(3.4,2.1)
\pscircle(3.5,2.1){0.1}
\psline[linewidth=1pt]{}(3.6,2.1)(3.9,2.1)
\pscircle(4.0,2.1){0.1}
\psline[linewidth=1pt]{}(4.1,2.1)(4.4,2.1)
\pscircle(4.5,2.1){0.1}
%\psline[linewidth=1pt]{}(4.5,0)(4.5,2.0)
\psline[linewidth=1pt]{<->}(5,0)(5,2.1)

\rput[tl](-0.8,0.4){$a_1$}
\rput[tl](5.2,1.1){$a_2$}
\rput[tl](-0.3,1.1){$i\!=\!1$}
\rput[tl](4.4,2.5){$i\!=\!N$}
\rput[tl](-0.1,-0.5){$\underbrace{\hspace{2.8cm}}_{\mathcal{C}_1}$}
\rput[tl](2.4,-0.5){$\underbrace{\hspace{2.8cm}}_{\mathcal{C}_2}$}
\end{pspicture}
\caption{\label{fig_twoclusterchain}A clustered graph signal $x[i] = a_{1} \mathcal{I}_{\mathcal{C}_{1}}[i] +a_{2} \mathcal{I}_{\mathcal{C}_{2}}[i]$  
(cf.\ \eqref{equ_def_clustered_signal_model}) defined over a chain graph $\graph_{\rm chain}$ which is partitioned into two equal-size clusters $\cluster_{1}$ and $\cluster_{2}$ which 
consist of consecutive nodes. 
The edges connecting nodes within the same cluster have weight $1$, while the single edge connecting nodes from different clusters has weight $1/\kappa$.}
\end{figure} 

\subsection{Recovery of Approximately Clustered Signals}

The scope of Theorem \ref{main_thm_exact_sparse} is somewhat limited as 
it applies only to graph signals which are precisely of the form \eqref{equ_def_clustered_signal_model}. 
We now state a more general result applying to any graph signal $\xsig\!\in\!\graphsigs$. 
\begin{theorem} 
\label{main_thm_approx_sparse}
Consider a graph signal $\xsig\!\in\!\graphsigs$ which is 
observed only at the sampling set $\samplingset$. 
If NNSP-$\partition$ holds with $\kappa =2$ in \eqref{equ_def_flows_NNSP}, any solution $\hat{\vx}$ 
of \eqref{equ_semi_sup_learning_problem} satisfies (cf.\ \eqref{equ_def_clustered_signal_model})
\begin{equation}
\label{equ_result_stable_NNSP}
\|\hat{x}[\cdot]- x[\cdot] \|_{\rm TV} \leq 6  \min_{\mathbf{a} \in \mathbb{R}^{|\partition|} } \| \xsig - \sum_{\cluster \in \partition} a_{\cluster} \mathcal{I}_{\cluster}[\cdot] \|_{\rm TV}. 
\end{equation} 
\end{theorem} 
Thus, as long as the underlying graph signal $x[\cdot]$ can be well approximated 
by a clustered signal of the form \eqref{equ_def_clustered_signal_model}, 
any solution $\hat{x}[\cdot]$ of \eqref{equ_semi_sup_learning_problem} is a graph signal 
which varies significantly only over the boundary edges $\partial \partition$. 
We highlight that the error bound \eqref{equ_result_stable_NNSP} only controls the TV (semi-)norm of 
the error signal $\hat{x}[\cdot] - x[\cdot]$. In particular, this bound does not directly allow to 
quantify the size of the global mean squared error $(1/\siglen) \sum_{i \in \nodes} (\hat{x}[i] - x[i])^2$. 
However, the bound \eqref{equ_result_stable_NNSP} allows to characterize idenfiability of the underlying partition $\partition$. 
Indeed, if the signal values $a_{\cluster}$ in \eqref{equ_def_clustered_signal_model} satisfy 
$\min_{\cluster \in \partition} |a_{\cluster} - a_{\cluster'}| \min_{e \in \partial \partition} W_{e} \geq \|\hat{x}[\cdot]- \hat{x}[\cdot] \|_{\rm TV}$, 
we can read off the cluster boundaries from the signal differences $\hat{x}[i] - \hat{x}[j]$ (over edges $\{i,j\} \in \edges$). 

One particular use of Theorems \ref{main_thm_exact_sparse}, \ref{main_thm_approx_sparse} is to guide the 
choice for the sampling set $\samplingset$. In particular, % for a suitably chosen partition $\partition$ and associated 
%signal model \eqref{equ_def_clustered_signal_model}, 
one should aim at sampling nodes such that the NNSP is likely to be satisfied. According to the definition of the NNSP, 
we should sample nodes which are well connected (in the sense of allowing for a large flow) to the boundary edges which 
connect different clusters. This approach has been studied empirically in \cite{MaraJungAsilomar2017,SaeedSampta17}, 
verifying accurate recovery by efficient convex optimization methods using sampling sets which satisfy the NNSP (cf. Definition \ref{def_sampling_set_resolves}) with 
high probability.   

%%%%%%%%%%%%%%%%%%%%%%%%%%
\section{Numerical Experiments}
\label{sec5_experiment}

We now verify the practical relevance of our theoretical findings by means of two numerical experiments. 
The first experiment is based on a synthetic data set whose underlying data graph is a chain graph 
$\graph_{\text{chain}}$. A second experiment revolves around a real-world data set describing the 
roadmap of Minnesota \cite{ChenVarma2016,Gleich2018}. 

\subsection{Chain Graph}
We generated a synthetic data set whose data graph is a chain graph $\graph_{\rm chain}$. 
This chain graph contains $|\nodes|=100$ nodes which are connected by $|\edges|=99$ undirected edges $\{i,i+1\}$, for $i \in \{1,\ldots,99\}$ and 
partitioned into $|\partition|=10$ equal-size clusters $\partition =\{ \cluster_l\}_{l=1,2,...,10}$, 
each cluster containing $10$ consecutive nodes. The edges connecting nodes in the same cluster 
have weight $W_{i,j}=4$, while those connecting different clusters have weight $W_{i,j}=2$. 
For this data graph we generated a clustered graph signal $x[i]$ of the form with alternating coefficients $a_{l} \in \{1,5\}$. 
%weight $W_{\{i,j\}} = 4$. The boundary edges $\{i,j\} \in \partial{\mathcal{F}}$, i.e. the edges connecting nodes $i$ 
%and $j$ which belong to different clusters, are weighted according to $W_{\{i,j\}} = 2$. This graph representation 
%of the signal thus incorporates the NNSP condition as denoted by Lemma \ref{lem_suff_cond_NNSP}. 
%The generated signal is of the form (\ref{equ_def_clustered_signal_model}) where we adopted coefficients 
%$a_l \in \{1,5\}$ representing the true signal values. 

The graph signal $x[i]$ is observed only at the nodes belonging to a sampling set, which is either $\samplingset_{1}$ or $\samplingset_{2}$. 
The sampling set $\samplingset_{1}$ contains exactly one node from each cluster $\cluster_{l}$ and thus, as can be verified easily, 
satisfies the NNSP (cf. Definition \ref{def_sampling_set_resolves}). While having the same size as $\samplingset_{1}$, the sampling set $\samplingset_{2}$ does not 
contain any node of clusters $\cluster_{2}$ and $\cluster_{4}$. 

In Figure \ref{fig:chaingraph_slpnnsp}, we illustrate the recovered signals obtained by solving \eqref{equ_semi_sup_learning_problem} 
using the sparse label propagation (SLP) algorithm \cite{JungHero2016}, which is fed with signal values on the sampling set (being either $\samplingset_{1}$ or $\samplingset_{2}$). 
The signal recovered from the sampling set $\samplingset_{1}$, which satisfies the NNSP, closely resembles the true underlying 
clustered graph signal. In contrast, the sampling set $\samplingset_{2}$, which does not satisfy the NNSP, results in a recovered signal 
which significantly deviates from the true signal.

\begin{figure}[h]
	\centering
	\vspace{-0.35cm}
	\includegraphics[width=\linewidth]{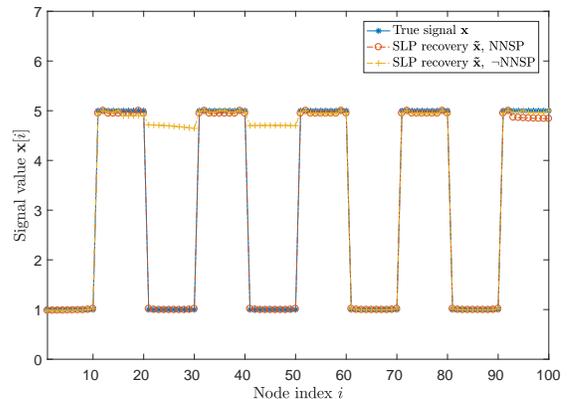}
	%\vspace{-0.5cm}
	\caption{Clustered graph signal $x[i]$ along with the recovered signals obtained from sampling sets $\samplingset_{1}$ and 
	$\samplingset_{2}$.}
	\label{fig:chaingraph_slpnnsp}
\end{figure}

\subsection{Minnesota Roadmap} 
The second data set, with associated data graph $\graph_{\rm min}$, represents the roadmap of Minnesota \cite{ChenVarma2016}. 
The data graph $\graph_{\rm min}$ consists of $|\nodes|=2642$ nodes, and $|\edges|=3303$ edges. We generate a 
clustered graph signal defined over $\graph_{\rm min}$ by randomly selecting three different nodes $\{i_{1},i_{2},i_{3}\}$ 
which are declared as ``cluster centres'' of the clusters $\cluster_{1},\cluster_{2},\cluster_{3}$. The remaining nodes 
$\nodes \setminus \{i_{1},i_{2},i_{3}\}$ are then associated to the cluster whose centre is nearest in the sense of smallest 
geodesic distance. The edges connecting nodes within the same cluster have weight $W_{i,j}=4$, 
and those connecting different clusters have weight $W_{i,j}=2$.

We use SLP to recover the entire graph signal from its values obtained for the nodes in a sampling set. Two different choices $\samplingset_{1}$ and $\samplingset_{2}$ 
for the sampling set are considered: The sampling set $\samplingset_{1}$ is based on the NNSP and consists of all nodes which are 
adjacent to the boundary edges between two different clusters. In contrast, the sampling set $\samplingset_{2}$ is obtained by 
selecting uniformly at random a total of $|\samplingset_{1}|$ nodes from $\graph_{\rm min}$, i.e., we ensure $|\samplingset_{1}| = |\samplingset_{2}|$. 

The resulting MSE is $0.0023$ for the sampling set $\samplingset_{1}$ (conforming with NNSP), while the recovery using the random sampling set $\samplingset_{2}$ 
incurred an average (over $100$ i.i.d. simulation runs) MSE of $0.0502$. %with standard deviation 0.0527 for sampling set $\samplingset_{2}$ over 100 runs, i.e. 100 different random sampling sets $\samplingset_{2}$.
In Figure \ref{fig:results_minnesota}, we depict the recovered graph signals using signal samples from either $\samplingset_{1}$ or $\samplingset_{2}$ (one typical realization). 
Evidently, the recovery using the sampling set $\samplingset_{1}$ (which is guided by the NNSP) results in a more accurate recovery compared to using the random sampling 
set $\samplingset_{2}$.

\begin{figure}
	\centering
	\vspace{-0.25cm}
	\subfloat[\label{fig:actual_minnesota}]{\includegraphics[scale=.45]{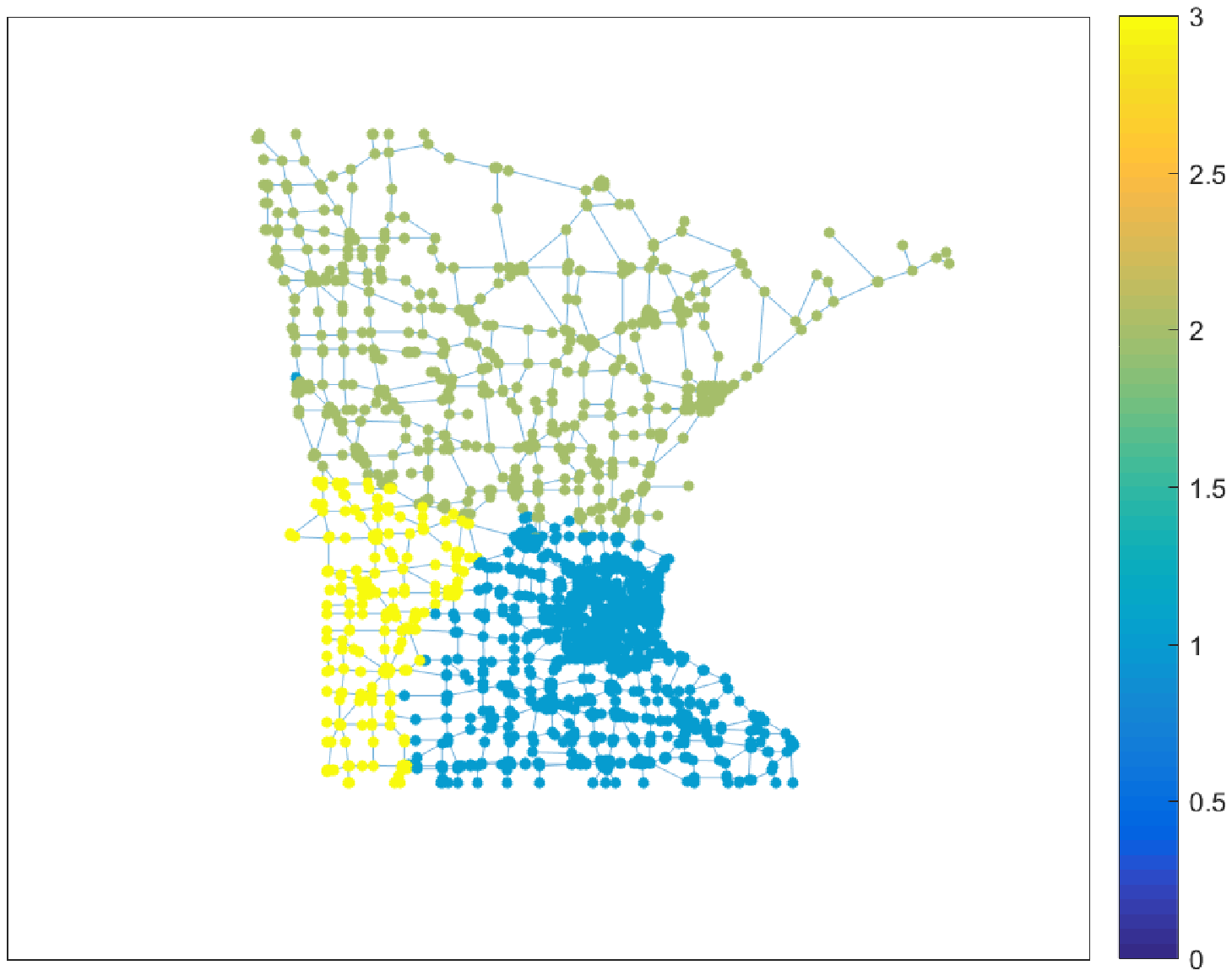}}
	\par\medskip
	\begin{minipage}{.5\linewidth}
		\centering
		\subfloat[\label{fig:minnesota_slpnnsp_1}]{\includegraphics[scale=.25]{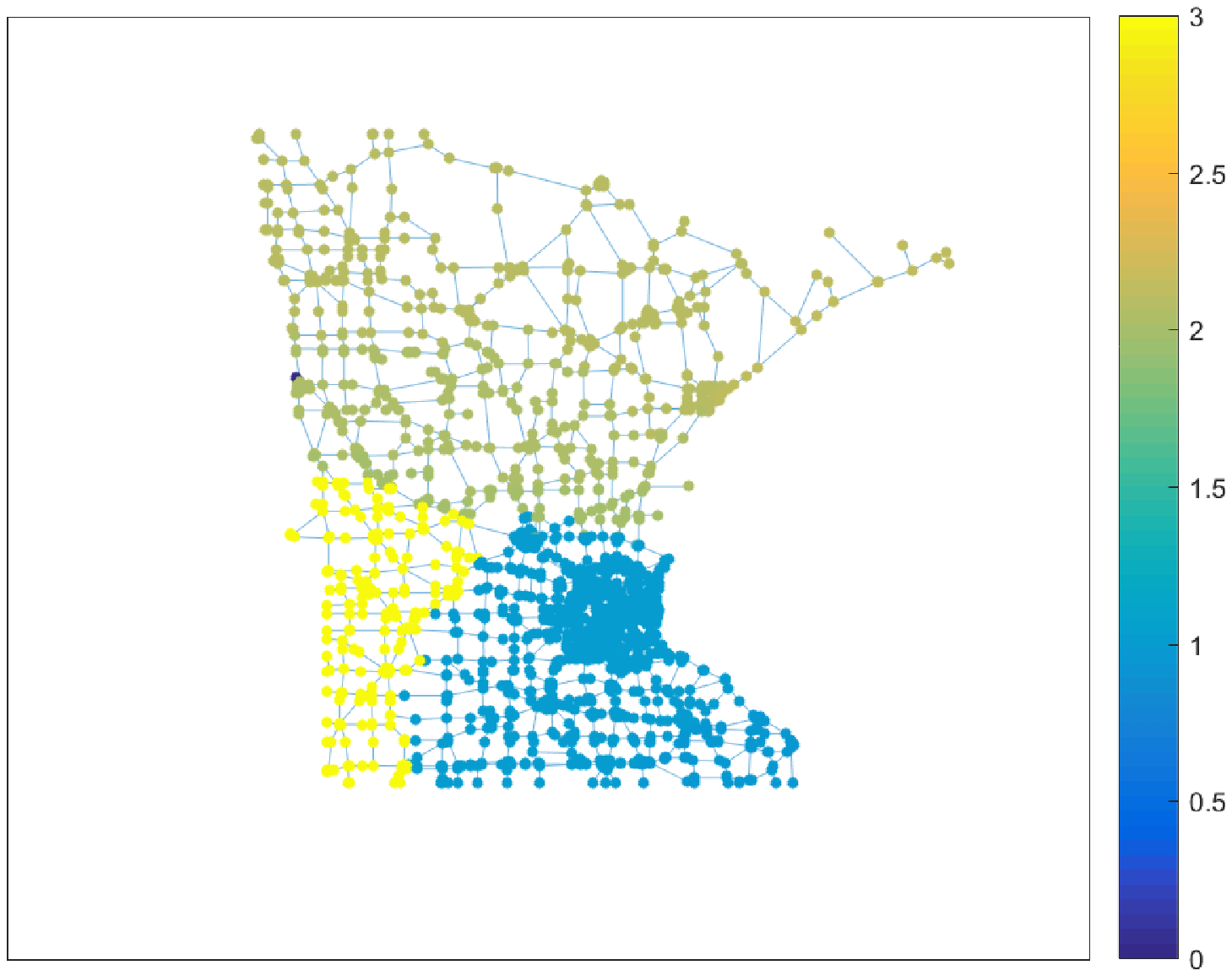}}
	\end{minipage}%
	\begin{minipage}{.5\linewidth}
		\centering
		\subfloat[\label{fig:minnesota_slpnnsp_2}]{\includegraphics[scale=.25]{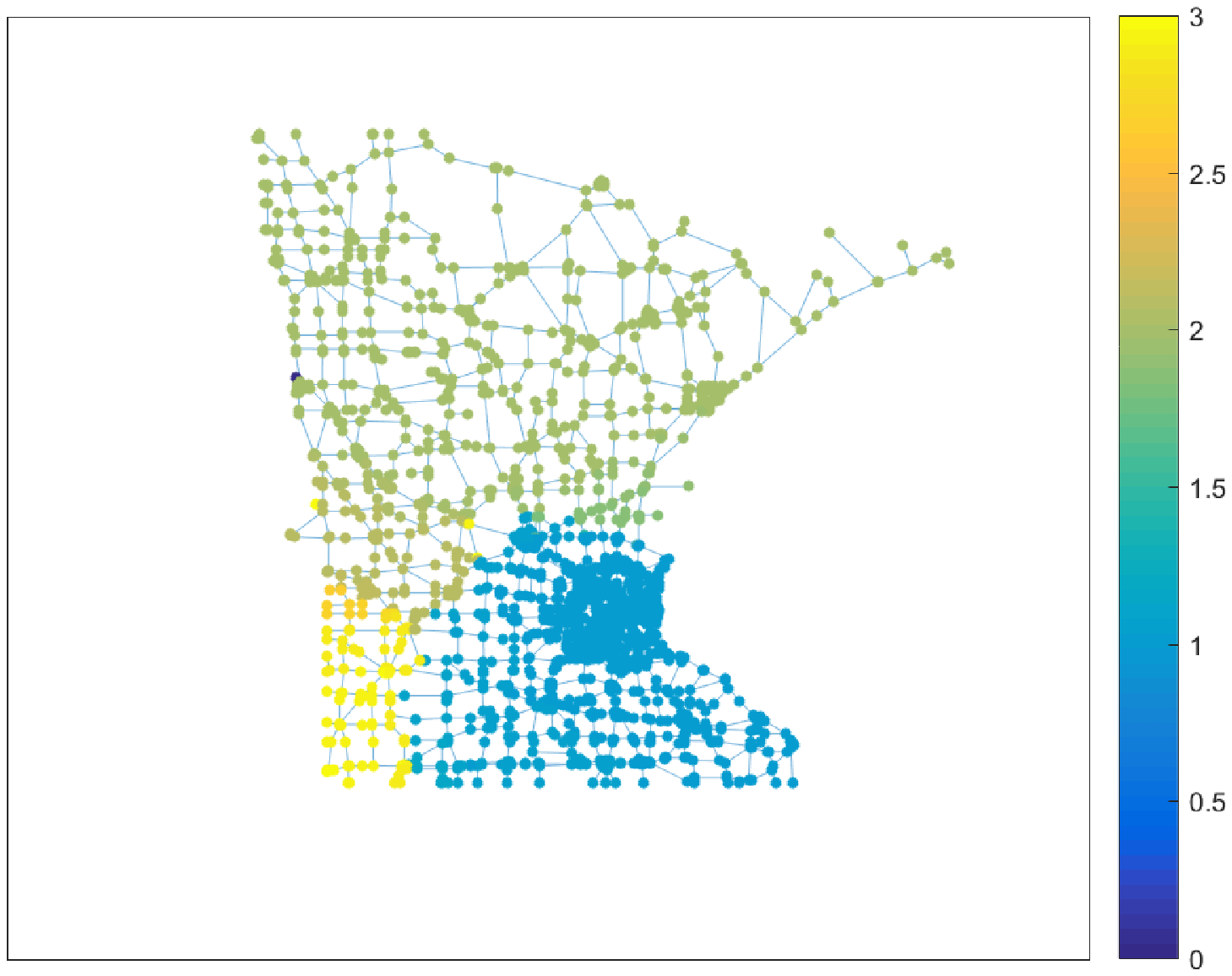}}
	\end{minipage}
	
	\caption{\label{fig:results_minnesota} True graph signal (a) and recovered signals for the Minnesota roadmap 
	data set obtained (b) when using sampling set $\samplingset_{1}$, (c) or sampling set $\samplingset_{2}$ (Figure \ref{fig:minnesota_slpnnsp_2}).}
	
	\vspace{-0.25cm}
\end{figure}

%%%%%%%%%%%%%%%%%%%%%%%%%%

\section{Conclusions}
\label{sec6_conclusion}

We considered the problem of recovering clustered graph signals, defined 
over complex networks, from observing its signal values on a small set of sampled nodes. 

By extending tools from compressed sensing, we derived a sufficient condition, 
the network nullspace property (NNSP), on the graph topology and sampling set such that a 
convex recovery method is accurate. This condition is based on the connectivity properties of the 
underlying network. In particular, it requires the existence of certain network 
flows with the edge weights of the data graph being interpreted as capacities.  

The NNSP involves both, the sampling set and the cluster structure of the data graph. 
Roughly speaking it requires to sample more densely near the boundaries between different clusters. 
This intuition has be verified by means of numerical experiments on synthetic and real-world datasets.

Our work opens up several avenues for future research. In particular, it would be interesting to analyze how likely the NNSP holds 
for certain random network models and sampling strategies. The tightness of the resulting recovery guarantees 
could then be contrasted with fundamental lower bounds obtained from an information-theoretic approach to 
minimax-estimation. Moreover, we would like to study variations of the SLP problem which are more suitable for classification problems.

\section*{Acknowledgement}
Parts of the material underlying this work has been presented in \cite{NNSPSampta2017}. 
However, this paper significantly extends \cite{NNSPSampta2017} by providing detailed proofs as well 
as extended discussions of the main results.
% and at the 2018 IEEE 
%International Conference on Acoustics, Speech and Signal Processing. 
This manuscript is available as a pre-print 
\cite{NNSPPreprint} at the following address: \url{https://arxiv.org/abs/1705.04379}. Copyright of this pre-print version rests with the authors.

%\section{Acknowledgment} 
%The author is grateful for valuable feedback on the manuscript by Ayelet Heimowitz and Yonina C.\ Eldar.
%\vspace*{-3mm}
\section{Proofs}
\label{sec_proofs} 
%\vspace*{-2mm}

The proofs for Theorem \ref{main_thm_exact_sparse} and Theorem \ref{main_thm_approx_sparse} rely on 
recognizing the recovery problem \eqref{equ_semi_sup_learning_problem} as an analysis 
$\ell_{1}$-minimization problem \cite{CoSparseModel}. %as it arises naturally 
%from the co-sparse analysis model of compressed sensing . 
A sufficient condition for analysis $\ell_{1}$-minimization to deliver 
the correct solution $x[\cdot]$ is given by the analysis nullspace 
property \cite{CoSparseModel,KabRau2015Chap}. 
%which can be considered as the analysis-model pendant of the nullspace property \cite[Def.\ 4.1.]{RauhutFoucartCS}. 
%The main technical part of the proof amounts to verifying 
%that the network nullspace property (cf.\ Definition \ref{def_sampling_set_resolves}) implies 
%this analysis nullspace property. 
%\subsection{The Analysis Nullspace Property}
%As an intermediate step towards proving our main 
%results Theorem \ref{main_thm_exact_sparse} and \ref{main_thm_approx_sparse}, 
%we now rephrase the stable analysis nullspace property \cite{CoSparseModel,KabRau2015Chap}
%in graph signal terminology. 
In particular, the sampling set $\samplingset$ is said to satisfy the stable analysis nullspace property w.r.t. an edge set 
$\edgesupport \subseteq \edges$ if 
\begin{equation} 
\label{equ_NSP1}
\|  u[\cdot] \|_{\edges \setminus \edgesupport} \geq \kappa \|  u[\cdot] \|_{\edgesupport} \mbox{ for any } u[\cdot] \in \mathcal{K}(\samplingset),
\end{equation} 
for some constant $\kappa >1$. 
\begin{lemma} 
\label{lem_NSP1}
Consider a data graph $\graph$ and fixed partitioning $\partition=\{\cluster_{1},\ldots,\cluster_{|\partition|}\}$ 
of its nodes into $|\partition|$ clusters $\cluster_{l}$. We observe a clustered graph 
signal $\xsig$ with $\xsigval{i} = \sum_{l=1}^{|\partition|} a_{l} \mathcal{I}_{\cluster_{l}} [i]$ 
at the sampled nodes $\samplingset \!\subseteq\!\nodes$.  
If \eqref{equ_NSP1} holds for $\edgesupport\!=\!\partial \partition$, 
then \eqref{equ_semi_sup_learning_problem} has a unique solution given by $\xsig$. 
\end{lemma}
\begin{proof}
Consider a graph signal $\hat{x}[\cdot]$, which is different from the true underlying graph signal $x[\cdot]$, being feasible for 
\eqref{equ_semi_sup_learning_problem}, i.e, $\hat{x}[i] = x[i]$ for all sampled nodes $i \in \samplingset$. Then, 
the difference $u[i] \defeq \hat{x}[i] - x[i]$ belongs to the kernel $\mathcal{K}(\samplingset)$ (cf.\ \eqref{equ_def_kernel_sampling_set}). 
Note that, since $x[i]$ is constant for all nodes $i \in \cluster_{l}$ in the same cluster, 
\begin{equation}
\label{equ_supp_vx_edgesupport}
\hat{x}[i] - \hat{x}[j] = u[i] - u[j] \mbox{, for any edge } \{i,j\} \in \edges \setminus\partial \partition. 
\end{equation} 
By the triangle inequality, 
\begin{align}
\| \hat{x}[\cdot] \|_{\partial \partition} & \geq \| x[\cdot] \|_{\partial \partition} - \| u[\cdot] \|_{\partial \partition}   =  \| x[\cdot] \|_{\rm TV}\!-\!\| u[\cdot] \|_{\partial \partition}, \nonumber
\end{align} 
and, since $\| \hat{x}[\cdot] \|_{\rm TV}  = \| \hat{x}[\cdot] \|_{\partial \partition} + \|\hat{x}[\cdot] \|_{\edges \setminus \partial \partition}$, in turn
\vspace*{0mm}
\begin{align}
\| \hat{x}[\cdot] \|_{\rm TV}  & \!=\! \|\hat{x}[\cdot] \|_{\partial \partition}\!+\!\|\hat{x}[\cdot]  \|_{\edges \setminus \partial \partition} 
\stackrel{\eqref{equ_supp_vx_edgesupport}}{=}\|\hat{x}[\cdot] \|_{\partial \partition} + u[\cdot] \|_{\edges \setminus \partial \partition}  \nonumber \\ 
 &  \hspace*{-8mm} \geq  \| x[\cdot] \|_{\rm TV}\!-\!\| u[\cdot]  \|_{\partial \partition}\!+\!\|u[\cdot] \|_{\edges \setminus \partial \partition} \stackrel{\eqref{equ_NSP1}}{>} \|x[\cdot] \|_{\rm TV}.  \nonumber \\[0mm]
 \nonumber 
\end{align}
Thus, we have shown that any graph signal $\hat{x}[\cdot]$ which is different from the true underlying graph signal $x[\cdot]$ but 
coincides with it at all sampled nodes $i \in \samplingset$, must have a larger TV norm than the true signal $x[\cdot]$ and 
therefore cannot be optimal for the problem \eqref{equ_semi_sup_learning_problem}.
\end{proof}
The next result extends Lemma \ref{lem_NSP1} to graph signals $\xsig \in \graphsigs$ which 
are not exactly clustered, but which can be well approximated by 
a clustered signal of the form \eqref{equ_def_clustered_signal_model}. 
\begin{lemma} 
\label{thm_approx_sparse_stable_result}
Consider a data graph $\graph$ and a fixed partition $\partition=\{\cluster_{1},\ldots,\cluster_{|\partition|}\}$ 
of its nodes into disjoint clusters $\cluster_{l} \subseteq \nodes$. We observe a graph signal $\vx \!\in\! \graphsigs$ at the sampling 
set $\samplingset \!\subseteq\! \nodes$. If \eqref{equ_NSP1} holds for $\edgesupport\!=\!\partial \partition$ and $\kappa\!=\!2$, any solution $\hat{x}[\cdot]$ of 
\eqref{equ_semi_sup_learning_problem} satisfies
\vspace*{-2mm}
\begin{equation} 
\label{equ_bound_stability}
\| x[\cdot]\!-\!\hat{x}[\cdot] \|_{\rm TV} \!\leq\! 6 \min_{a_{l} \in \mathbb{R}}  \big\| x[\cdot]  - \sum_{l = 1}^{|\partition|} a_{l} \mathcal{I}_{\cluster_{l}}[\cdot] \big\|_{\rm TV}.
\vspace*{-2mm}
\end{equation} 
\end{lemma} 
\begin{proof}
The argument closely follows the proof of \cite[Theorem 8]{KabRau2015}. 
First note that any solution $\hat{x}[\cdot]$ of \eqref{equ_semi_sup_learning_problem} obeys 
\begin{equation}
\label{equ_D_hat_x_1_norm_smaller}
\|  \hat{x}[\cdot] \|_{\rm TV} \leq \| x[\cdot] \|_{\rm TV}, 
\end{equation}  
since $x[\cdot]$ is trivially feasible for \eqref{equ_semi_sup_learning_problem}. 
From \eqref{equ_D_hat_x_1_norm_smaller}, we have  
\begin{equation}
\label{equ_proof_stable_NSP_123}
\hspace*{-3mm}\| \hat{x}[\cdot] \|_{\edgesupport} \!+\!\| \hat{x}[\cdot] \|_{\edges \setminus \edgesupport}  \!\leq\! \|  x[\cdot]  \|_{\edgesupport}\!+\!\| x[\cdot]  \|_{\edges \setminus \edgesupport}. 
\end{equation}  
Since $\hat{x}[\cdot]$ is feasible for \eqref{equ_semi_sup_learning_problem}, i.e., $\hat{x}[i] = x[i]$ for every sampled node $i \in \samplingset$, 
the difference $v[\cdot] \defeq \hat{x}[\cdot]\!-\!x[\cdot]$ belongs to $\mathcal{K}(\samplingset)$ (cf.\ \eqref{equ_def_kernel_sampling_set}). 
Applying the triangle inequality to \eqref{equ_proof_stable_NSP_123}, 
%\begin{align} 
%& \|  x[\cdot]  \|_{\edgesupport}\!-\!\|  v[\cdot]  \|_{\edgesupport}\!-\!\|  x[\cdot]  \|_{\edges \setminus \edgesupport}\!+\!\| v[\cdot]  \|_{\edges \setminus \edgesupport} \leq \nonumber \\ 
%&   \|   x[\cdot]  \|_{\edgesupport}\!+\!\|   x[\cdot]  \|_{\edges \setminus \edgesupport}, 
%\end{align} 
%and, in turn, 
\begin{align} 
\label{equ_DM_vv_mD_vv}
\| v[\cdot] \|_{\edges \setminus \edgesupport}  \leq \| v[\cdot]  \|_{\edgesupport} + 2 \| x[\cdot] \|_{\edges \setminus \edgesupport} . 
\end{align} 
Combining \eqref{equ_DM_vv_mD_vv} with \eqref{equ_NSP1} (for the signal $u[\cdot] = v[\cdot]$), 
\begin{equation}
\label{equ_offsupport_leq_fourtimes}
\|  v [\cdot] \|_{\edges \setminus \edgesupport}  \leq 4 \|  x[\cdot]  \|_{\edges \setminus \edgesupport}. 
\end{equation} 
Using \eqref{equ_NSP1} again, 
\begin{align}
\| x[\cdot]\!-\!\hat{x}[\cdot] \|_{\rm TV} & = \|  v[\cdot] \|_{\rm TV}  = \|  v[\cdot]  \|_{\edgesupport}\!+\!\|  v[\cdot] \|_{\edges \setminus \edgesupport}  \nonumber \\
& \stackrel{\eqref{equ_NSP1}}{\leq} (3/2) \!\|  v[\cdot]  \|_{\edges \setminus \edgesupport} \stackrel{\eqref{equ_offsupport_leq_fourtimes}}{\leq} 6  \|  x[\cdot]  \|_{\edges \setminus \edgesupport}. \nonumber
\end{align}
For any clustered graph signal $x_{\rm c}[\cdot]$ of the form $x_{\rm c}[i]= \sum_{l = 1}^{|\partition|} a_{l} \mathcal{I}_{\cluster_{l}}[i]$, 
we have $x_{c}[i]\!-\!x_{c}[j]=0$ for any $\{i,j\} \in \edges \setminus \edgesupport$ (note that $\edgesupport=\partial \partition$)  and, in turn, 
\begin{align}
\|x[\cdot]\!+\!x_{\rm c}[\cdot]\|_{\rm TV} &=\!\|\xsig\!+\!x_{\rm c}[\cdot]  \|_{\edges \setminus \edgesupport}+\|\xsig\!+\!\xsig_{c} \|_{\edgesupport}  \nonumber \\
& \geq \| \xsig\!+\!\xsig_{c}  \|_{\edges \setminus \edgesupport} = \| \xsig \|_{\edges \setminus \edgesupport}.  \nonumber%\\[-10mm]
%\nonumber
\end{align} 
\end{proof} 
Let us now render Lemma \ref{lem_NSP1} and Lemma \ref{thm_approx_sparse_stable_result} for clustered 
graph signals $\xsig$ of the form \eqref{equ_def_clustered_signal_model} by stating a condition on the graph 
topology and sampling set $\samplingset$ which ensures \eqref{equ_NSP1}. 

\begin{lemma} 
\label{lem_NNSP_samplingset_suff_recovery}
If a sampling set $\samplingset$ satisfies NNSP-$\partition$, then
it also satisfies the stable analysis nullspace property \eqref{equ_NSP1}.  
\end{lemma} 
\begin{proof} 
Consider a signal $u[\cdot] \!\in\! \mathcal{K}(\samplingset)$ which vanishes at all sampled nodes, i.e., 
%\begin{equation}
%\label{equ_sig_in_kernel_samplingset}
$u[i] = 0 \mbox{ for all } i \in  \samplingset$.
%\end{equation} 
We will now show that $\| u[\cdot] \|_{\edges \setminus \partial\partition} \geq 2 \| u[\cdot] \|_{\partial\partition}$.

Let us assume that for each boundary edge $e \in \partial \partition$, the flow $f[\cdot]$ in Definition \ref{def_sampling_set_resolves} has 
the same sign as $u[e^{+}]- u[e^{-}]$. We are allowed to assume this since according to Definition \ref{def_sampling_set_resolves}, 
if there exists a flow with $f[e'] > 0$ for a boundary edge $e' \in \partial \partition$, there is another flow $\tilde{f}[\cdot]$ with $\tilde{f}[e'] = - f[e']$ 
for the same edge $e' \in \partial \partition$, but otherwise identical to $f[\cdot]$, i.e., $ \tilde{f}[e]=f[e]$ for all $e \in \edges \setminus \{ e' \}$. 

Next, we construct an augmented graph $\graph'$ by adding an extra node $s$ to the data graph $\graph$ which is connected to all sampled nodes $i \in \samplingset$ 
via an edge $e_{i} = \edge{s}{i}$ which is oriented such that $e_{i}^{+} =s$. We assign to each edge $e_{i} = \edge{s}{i}$ the 
flow $f[e_{i}] = g[i]$ (cf. \eqref{equ_define_demand_now}). It can be verified easily that the flow over the augmented graph has zero demands for all nodes. Thus, 
we can apply Tellegen's theorem \cite{Tellegen1952} to obtain $\| u[\cdot] \|_{\edges \setminus \partial\partition} \geq 2 \| u[\cdot] \|_{\partial\partition}$. 

\vspace*{-4mm}
\end{proof} 

We obtain Theorem \ref{main_thm_exact_sparse} by combining Lemma \ref{lem_NNSP_samplingset_suff_recovery} with Lemma \ref{lem_NSP1}. 
In order to verify Theorem \ref{main_thm_approx_sparse} we note that, by Lemma \ref{lem_NNSP_samplingset_suff_recovery}, the NNSP according to 
Definition \ref{def_sampling_set_resolves} implies the stable nullspace property \eqref{equ_NSP1} for $\edgesupport=\partial \partition$. 
Therefore, we can invoke Lemma \ref{thm_approx_sparse_stable_result} 
to reach \eqref{equ_result_stable_NNSP}.

\bibliographystyle{abbrv}
%\bibliographystyle{plain}
%\bibliography{Sampta2017}
%\bibliography{NNSP_JMLR}
\bibliography{SLPBib}

\begin{thebibliography}{10}

\bibitem{SaeedSampta17}
S.~Basirian and A.~Jung.
\newblock Random walk sampling for big data over networks.
\newblock In {\em Proc. Int. Conf. Sampling. Th. and App. (SampTA)}, pages
  427--431, Tallinn, Estonia, Jul. 2017.

\bibitem{pock_chambolle}
A.~Chambolle and T.~Pock.
\newblock A first-order primal-dual algorithm for convex problems with
  applications to imaging.
\newblock {\em J. Math. Imaging Vision}, 40(1):120--145, May 2011.

\bibitem{SemiSupervisedBook}
O.~Chapelle, B.~Sch{\"o}lkopf, and A.~Zien, editors.
\newblock {\em Semi-Supervised Learning}.
\newblock The MIT Press, Cambridge, Massachusetts, 2006.

\bibitem{ChenVarma2015}
S.~Chen, R.~Varma, A.~Sandryhaila, and J.~{Kova{\v c}evi{\'c}}.
\newblock Discrete signal processing on graphs: sampling theory.
\newblock {\em IEEE Trans. Signal Processing}, 63(24):6510--6523, Dec. 2015.

\bibitem{ChenVarma2016}
S.~Chen, R.~Varma, A.~Sandryhaila, and J.~{Kova{\v c}evi{\'c}}.
\newblock Representations of piecewise smooth signals on graphs.
\newblock In {\em Proc. IEEE ICASSP 2016}, Shanghai, China, March 2016.

\bibitem{BigDataNetworksBook}
S.~Cui, A.~Hero, Z.-Q. Luo, and J.~Moura, editors.
\newblock {\em Big Data over Networks}.
\newblock Cambridge Univ. Press, 2016.

\bibitem{DiestelGT}
R.~Diestel.
\newblock {\em Graph Theory}.
\newblock Springer, New York, 2000.

\bibitem{Don06}
D.~L. Donoho.
\newblock Compressed sensing.
\newblock {\em IEEE Trans. Inf. Theory}, 52(4):1289--1306, Apr. 2006.

\bibitem{EldarBookSamplingTheory}
Y.~C. Eldar.
\newblock {\em Sampling Theory - Beyond Bandlimited Systems}.
\newblock Cambridge Univ. Press, Cambridge, UK, 2015.

\bibitem{Fortunato2009}
S.~Fortunato.
\newblock Community detection in graphs.
\newblock {\em Physics Reports}, 486:75--174, Feb. 2010.

\bibitem{RauhutFoucartCS}
S.~Foucart and H.~Rauhut.
\newblock {\em A Mathematical Introduction to Compressive Sensing}.
\newblock Springer, New York, 2012.

\bibitem{Gleich2018}
D.~Gleich.
\newblock {MatlabBGL} - {A} matlab graph library, 2007.

\bibitem{HannakAsilomar2016}
G.~Hannak, P.~Berger, G.~Matz, and A.~Jung.
\newblock Efficient graph signal recovery over big networks.
\newblock In {\em Proc. Asilomar Conf. Signals, Systems, Computers}, pages
  1--6, Pacific Grove, CA, Nov. 2016.

\bibitem{JungSpawc2016}
A.~Jung, P.~Berger, G.~Hannak, and G.~Matz.
\newblock Scalable graph signal recovery for big data over networks.
\newblock In {\em Proc. IEEE-SP Workshop on Sig. Proc. Adv. in Wireless Comm.
  (SPAWC)}, Edinburgh, UK, Jul. 2016.

\bibitem{NNSPSampta2017}
A.~Jung, A.~Heimowitz, and Y.~C. Eldar.
\newblock The network nullspace property for compressed sensing over networks.
\newblock In {\em Proc.\ Int. Conf. Sampling Th. and App. (SampTA)}, Tallinn,
  Estonia, Jul. 2017.

\bibitem{JungHero2016}
A.~{Jung}, A.~O. {Hero}, III, A.~{Mara}, and B.~{Jahromi}.
\newblock Semi-supervised learning via sparse label propagation.
\newblock {\em submitted to a journal. preprint available under
  https://arxiv.org/abs/1612.01414}, May 2017.

\bibitem{NNSPPreprint}
A.~Jung and M.~Hulsebos.
\newblock The network nullspace property for compressed sensing of big data
  over networks.
\newblock {\em arXiv}, 2017.

\bibitem{NlassoFrontiers2018}
A.~Jung, N.~Tran, and A.~Mara.
\newblock When is network lasso accurate?
\newblock {\em Frontiers in Applied Mathematics and Statistics}, 3:28, Jan.
  2018.

\bibitem{KabRau2015}
M.~Kabanava and H.~Rauhut.
\newblock Analysis $\ell_{1}$-recovery with frames and gaussian measurements.
\newblock {\em Acta Applicandae Mathematicae}, 140(1):173 -- 195, Dec. 2015.

\bibitem{KabRau2015Chap}
M.~Kabanava and H.~Rauhut.
\newblock Cosparsity in compressed sensing.
\newblock In H.~Boche, R.~Calderbank, G.~Kutyniok, and J.~Vybiral, editors,
  {\em Compressed Sensing and Its Applications}, pages 315--339. Springer,
  2015.

\bibitem{KleinbergTardos2006}
J.~Kleinberg and E.~Tardos.
\newblock {\em Algorithm Design}.
\newblock Addison Wesley, New York, 2006.

\bibitem{MaraJungAsilomar2017}
A.~Mara and A.~Jung.
\newblock Recovery conditions and sampling strategies for network lasso.
\newblock In {\em Proc. Asilomar Conf. Signals, Systems, Computers}, Pacific
  Grove, CA, Nov. 2017.

\bibitem{MaruesSeg2016}
A.~G. Marques, S.~Segarra, G.~Leus, and A.~Ribeiro.
\newblock Sampling of graph signals with successive local aggregation.
\newblock {\em IEEE Trans. Signal Processing}, 64(7):1832--1843, Apr. 2016.

\bibitem{CoSparseModel}
S.~Nam, M.~Davies, M.~Elad, and R.~Gribonval.
\newblock The cosparse analysis model and algorithms.
\newblock {\em Applied and Computational Harmonic Analysis}, 34(1):30--56, Jan.
  2013.

\bibitem{SharpnackJMLR2012}
J.~Sharpnack, A.~Rinaldo, and A.~Singh.
\newblock Sparsistency of the edge lasso over graphs.
\newblock In {\em Proc. 15th Int. Conf. on Artificial Intelligence and
  Statistics (AISTATS)}, La Palma, Canary Islands, Apr. 2012.

\bibitem{Spielman_alocal}
D.~Spielman and S.-H. Teng.
\newblock A local clustering algorithm for massive graphs and its application
  to nearly linear time graph partitioning.
\newblock {\em SIAM J. Comput.}, 42(1), Jan. 2013.

\bibitem{Tellegen1952}
B.~D.~H. Tellegen.
\newblock A general network theorem with applications.
\newblock {\em Philips Research Reports}, 7:259--269, Feb. 1952.

\bibitem{TrendGraph}
Y.-X. Wang, J.~Sharpnack, A.~J. Smola, and R.~J. Tibshirani.
\newblock Trend filtering on graphs.
\newblock {\em J. Mach. Lear. Research}, 17, Apr. 2016.

\bibitem{zhao2004wsn}
F.~Zhao and L.~J. Guibas.
\newblock {\em Wireless Sensor Networks: An Information Processing Approach}.
\newblock Morgan Kaufmann, Amsterdam, The Netherlands, 2004.

\bibitem{ZhuAugADMM}
Y.~Zhu.
\newblock An augmented {ADMM} algorithm with application to the generalized
  lasso problem.
\newblock {\em Journal of Computational and Graphical Statistics},
  26(1):195--204, Feb. 2017.

\end{thebibliography}

\end{document}